\pdfoutput=1 % MUST stay within the first 5 lines: tells arXiv to use pdfLaTeX.
\documentclass[11pt]{article}

\usepackage{pgfplots}
\pgfplotsset{compat=1.18}

\usepackage[font=libertinus, citestyle=numeric]{kurbanlab}

\DeclareAffiliation{hbku}{%
  College of Science and Engineering, Hamad Bin Khalifa University, Doha, Qatar}

\DeclareAffiliation{tamu}{%
  Department of Computer and Electrical Engineering,
  Texas A\&M University, College Station, TX, USA}

\DeclareAffiliation{iub}{%
  Luddy School of Informatics, Computing, and Engineering,
  Indiana University Bloomington, Bloomington, IN, USA}

\usepackage{makecell}
\usepackage{multirow}
\usetikzlibrary{positioning,arrows.meta,calc,fit,backgrounds,shapes.geometric}

\definecolor{wOrange}{HTML}{E69F00}
\definecolor{wSky}{HTML}{56B4E9}
\definecolor{wGreen}{HTML}{009E73}
\definecolor{wBlue}{HTML}{0072B2}
\definecolor{wVerm}{HTML}{D55E00}
\definecolor{wPurple}{HTML}{CC79A7}
\definecolor{wGrey}{HTML}{BBBBBB}

\newcommand{\cg}{\mathrm{CG}}
\newcommand{\hph}{\mathrm{HPH}}
\newcommand{\sfs}{\mathrm{SFS}}   % from the appendix preamble

\newcommand{\Grd}{\mathcal{G}}
\newcommand{\Bld}{\mathcal{B}}
\newcommand{\Cst}{\mathcal{C}}
\newcommand{\fn}[1]{\textsf{#1}}
\newcommand{\bt}{\boldsymbol{\theta}}
\newcommand{\bu}{\mathbf{u}}
\title{One Perturbation Is Not Enough: Identifiability and Blind Baselines for Behavioral AI Evaluation}
\RunningTitle{Identifiability and blind baselines for behavioral evaluation}

\Author{Rasul Khanbayov}{hbku}
\Author{Mariam Sohail}{hbku}
\Author{Ahmed Abdala}{hbku}
\Author[corresponding=hkurban@hbku.edu.qa, orcid=0000-0003-3142-2866]{Hasan Kurban}{hbku}
\CodeURL{https://github.com/KurbanIntelligenceLab/one-perturbation-is-not-enough}
\begin{document}
\maketitle

\begin{abstract}
Behavioral evaluations perturb an input and read the induced change in the
output in order to certify that a system uses that input. We show that the
number of perturbations such a certificate requires is fixed, and that
reporting a single perturbation cannot supply it. Where a response ratio is a
property of the policy rather than of the test items, the behavioral record is
a linear measurement of an exponent vector recording how much the output
depends on each input, so perturbations identify input use exactly when their
logarithms span the input space. At least $n$ are needed for $n$ inputs, an
incomplete design confuses precisely the policies differing along the kernel of
its design matrix, and sharpening a perturbation never substitutes for adding
an independent one. We also derive in closed form the score such
a test awards a policy that reads nothing, which is far from zero and which
none of the probes we survey reports. Instantiating this where the correct
response is fixed by dimensional analysis, we run a complete identifying set of
three perturbations on three vision--language models reporting a physical
quantity from video. All three score far below their own blind bound rather
than above it, because each defaults to one of a small set of round
calibration values that never matches what the scale asserts; none moves its
relabeling response by a single exponent, and none is separable from the same
model instructed to ignore the video.
\end{abstract}

%%%%%%%%%%%%%%%%%%%%%%%%%%%%%%%%%%%%%%%%%%%%%%%%%%%%%%%%%%%%%%%%%%%%%%%%%%%%%%%%
\section{Introduction}
\label{sec:intro}

% CONVERSION NOTE: was a two-column figure*; set as a one-column figure here.
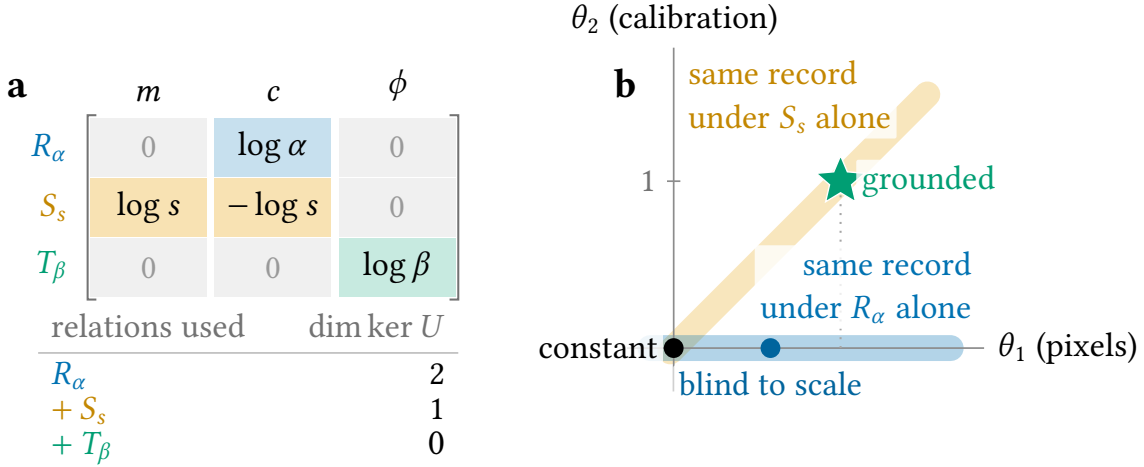
\begin{figure}[t]
\centering
\resizebox{\linewidth}{!}{%
\begin{tikzpicture}[font=\small,
  lbl/.style={fill=white,fill opacity=0.80,text opacity=1,inner sep=1.6pt,rounded corners=0.6pt}]

%% ================= a: the design matrix and the rank ledger =================
\begin{scope}[shift={(0,0)},
  cell/.style={minimum width=12.0mm,minimum height=6.0mm,inner sep=0pt},
  zero/.style={cell,fill=black!6,text=black!40},
  cR/.style={cell,fill=wBlue!22},
  cS/.style={cell,fill=wOrange!30},
  cT/.style={cell,fill=wGreen!22}]
  \node[anchor=south] at (0.64,0.96) {$m$};
  \node[anchor=south] at (1.92,0.96) {$c$};
  \node[anchor=south] at (3.20,0.96) {$\phi$};
  \node[anchor=east,text=wBlue]            at (-0.06,0.62) {$R_\alpha$};
  \node[anchor=east,text=wOrange!85!black] at (-0.06,0.00) {$S_s$};
  \node[anchor=east,text=wGreen]           at (-0.06,-0.62) {$T_\beta$};
  \node[zero] at (0.64, 0.62) {$0$};      \node[cR]   at (1.92, 0.62) {$\log\alpha$};
  \node[zero] at (3.20, 0.62) {$0$};      \node[cS]   at (0.64, 0.00) {$\log s$};
  \node[cS]   at (1.92, 0.00) {$-\log s$};\node[zero] at (3.20, 0.00) {$0$};
  \node[zero] at (0.64,-0.62) {$0$};      \node[zero] at (1.92,-0.62) {$0$};
  \node[cT]   at (3.20,-0.62) {$\log\beta$};
  \draw[black!55,line width=0.6pt] (0.10,0.95) -- (0.01,0.95) -- (0.01,-0.95) -- (0.10,-0.95);
  \draw[black!55,line width=0.6pt] (3.74,0.95) -- (3.83,0.95) -- (3.83,-0.95) -- (3.74,-0.95);
  \node[anchor=base west,text=black!55] at (-0.50,-1.34) {relations used};
  \node[anchor=base east,text=black!55] at (3.83,-1.34) {$\dim\ker U$};
  \draw[black!25,line width=0.5pt] (-0.50,-1.50) -- (3.83,-1.50);
  \node[anchor=base west,text=wBlue]            at (-0.50,-1.82) {$R_\alpha$};
  \node[anchor=base east]                       at (3.83,-1.82) {$2$};
  \node[anchor=base west,text=wOrange!85!black] at (-0.50,-2.18) {$+\,S_s$};
  \node[anchor=base east]                       at (3.83,-2.18) {$1$};
  \node[anchor=base west,text=wGreen]           at (-0.50,-2.54) {$+\,T_\beta$};
  \node[anchor=base east,font=\bfseries]        at (3.83,-2.54) {$0$};
\end{scope}
\node[anchor=base,font=\bfseries\large] at (-0.72,1.14) {a};

%% ================= b: the spatial plane =====================================
\begin{scope}[shift={(6.05,-1.45)},x=1.72cm,y=1.72cm]
  \draw[wOrange,line width=7.6pt,opacity=0.26,line cap=round] (-0.02,-0.02) -- (1.52,1.52);
  \draw[wBlue,line width=7.6pt,opacity=0.30,line cap=round]  (-0.14,0) -- (1.66,0);
  \draw[black!45,line width=0.5pt] (-0.28,0) -- (1.86,0);
  \draw[black!45,line width=0.5pt] (0,-0.26) -- (0,1.80);
  \draw[black!45,line width=0.5pt] (0.04,1) -- (-0.04,1);
  \node[anchor=east,inner sep=2.5pt,text=black!55] at (-0.05,1) {$1$};
  \node[anchor=west,inner sep=2.5pt] at (1.88,0) {$\theta_1$ (pixels)};
  \node[anchor=south,inner sep=2.5pt] at (0,1.82) {$\theta_2$ (calibration)};
  \draw[black!35,dotted,line width=0.6pt] (1,0) -- (1,0.93);
  \node[star,star points=5,star point ratio=2.2,fill=wGreen,draw=white,line width=0.4pt,
        inner sep=2.4pt] at (1,1) {};
  \fill[black] (0,0) circle (2.9pt);
  \fill[wBlue!88!black] (0.58,0) circle (2.9pt);
  \node[lbl,anchor=east]  at (-0.06,0) {constant};
  \node[lbl,anchor=north,text=wBlue!88!black] at (0.58,-0.10) {blind to scale};
  \node[lbl,anchor=west,text=wGreen] at (1.09,1.00) {grounded};
  \node[lbl,anchor=north east,text=wBlue,align=right] at (1.82,0.62)
    {same record\\under $R_\alpha$ alone};
  \node[lbl,anchor=north west,text=wOrange!85!black,align=left] at (0.05,1.76)
    {same record\\under $S_s$ alone};
\end{scope}
\node[anchor=base,font=\bfseries\large] at (5.55,1.14) {b};
\end{tikzpicture}}
\caption{\textbf{How many perturbations an evaluation needs.}
\textbf{a}, Each relation is one row of the design matrix $U$; each row
independently shrinks the confusion set. \textbf{b}, Policies sharing a
record with $\bt$ form the coset $\bt+\ker U$ (\cref{thm:rank}).
Relabeling alone merges a scale-blind tracker with a constant emitter;
resampling alone merges grounded with constant instead; the bands cross only
at the origin. The spatial pair together still leaves the temporal
coordinate $\theta_3$ unconstrained, so a grounded policy and one that
ignores only the asserted frame rate remain indistinguishable until
$T_\beta$ is added.}
\label{fig:teaser}
\end{figure}

Evaluations decide what gets deployed, so an evaluation that cannot separate
a system exhibiting a property from one merely insensitive to the input does
not support the decision it is used for. This paper gives the exact conditions
under which one widely used class of behavioral test has that defect, the
minimal design change that removes it, and a demonstration that the defect
fires on current models. We work in a setting where the correct response is
fixed by dimensional analysis rather than by a judge, so the diagnosis rests on
arithmetic rather than on a rater. Benchmarks now ask vision--language models
(VLMs) to report physical quantities from video and score the predicted number
against ground truth
\citep{quantiphy2025,iris2026,physicsmind2026}, an advance over qualitative
visual question answering \citep{vqa2015,chow2025physbench}. Scoring the
number alone conflates two behaviors: a model that measures, by locating the
object, reading the scale, and converting pixels to physical units, and a
model that emits a plausible prior can receive the same score whenever the
prior lands in the tolerance band.

\citet{quantiphy2025} supply the sharpest existing test. They hand the model a
physical prior in the prompt, multiply it by a counterfactual factor $\alpha$,
and check whether the prediction tracks $\alpha y$. It does not: most of their
21 models lose about $80\%$ of their score and even the strongest loses
$70\%$. We ask whether that result survives when the calibration must instead
be read off the image. The distinction matters because the failures differ in
cause and in remedy. A model may ignore a textual prior because it distrusts
an implausible number, which \citeauthor{quantiphy2025} observe directly,
while retaining the ability to measure; a model that cannot read an in-frame
scale fails earlier, and no prompting repairs it.

We move the analytic target into the visual channel. Relabel the in-frame
scale so the true calibration changes from $c$ to $\alpha c$ millimeters per
pixel, leaving the object's pixel trajectory untouched. A faithful measurer's
reported velocity must change by exactly $\alpha$, because its pixel
measurement is unchanged and only the conversion moved. The target is
analytic, so no rater or judge is needed, and because the score is a ratio of
two model outputs it needs no ground-truth quantity at all.

This construction is a \emph{metamorphic relation}: a property linking the
outputs of two related inputs, checkable without a label, and the standard
response to the oracle problem in software testing
\citep{chen1998metamorphic,vqamt2021,mtvla2026}. Relations used for these
models are predominantly \emph{semantic-consistency} relations expecting no
change; ours is quantitative and expects a specific nonzero change fixed by
dimensional analysis. Building the probe is the easy part. Two questions
decide whether its output means anything, and neither is standard practice.
Both are construct-validity requirements \citep{bean2025construct}, and they
bind hardest where a probe certifies a property rather than ranks systems: a
test that cannot separate ``ignores this input'' from ``ignores everything''
licenses the wrong conclusion in both directions.

\paragraph{Why identifiability matters beyond this probe}
\citet{barnett2024evals} argue that behavioral evaluation cannot upper-bound a
capability, since failing to find one is not evidence of absence, and
\citet{iasr2026} report that many benchmarks fail to measure the capability
they name because models reach correct answers by shortcut. The failure here is
asymmetric: a change-expecting relation is violated by \emph{any} policy whose
output does not move, so a model emitting confident constants scores exactly as
one that reasons correctly but ignores a single input --- and both developers
and regulators have an incentive for an evaluation to understate what a system
does \citep{sandbagging2024}. \citet{noiseinject2024} detect this by adding a
second probe with a different expected response rather than by sharpening the
first; the structure is the same as here, and \cref{cor:mtpair}
explains why a different kind of probe is needed where a more sensitive
version of the first would not help.

\paragraph{Two questions this raises}
What does a response identify?
Relabeling alone is not enough, and \cref{fig:teaser} shows which pair
of policies it leaves merged and which second relation separates them. What
would an ungrounded model score? A calibration-grounding rate is
uninterpretable without the rate a policy that ignores the scale would
achieve, and that rate is not zero; we derive it exactly, as a supremum over
all calibration-blind policies (\cref{thm:blind}).

\begin{kilkey}
One perturbation is never enough. Certifying that a system uses all $n$ of its
inputs requires $n$ relations whose log-perturbation vectors span the input
space; sharpening a probe already in the span cannot substitute for adding an
independent one.
\end{kilkey}

\paragraph{Contributions}
\begin{enumerate}
\item \textbf{An exact identifiability theorem.} The behavioral record of a
perturbation-based evaluation is a linear measurement of an exponent vector.
Identification holds exactly when the log-perturbation vectors span the
input space, at least $n$ perturbations are needed for $n$ inputs, and an
incomplete design confuses exactly one coset of the design matrix's kernel
(\cref{prop:cauchy}, \cref{thm:rank}). This replaces case
analysis with a rank condition and tells a designer which confusions survive
when a perturbation cannot be added (\cref{cor:mtpair}).
\item \textbf{Closed-form blind baselines.} The accuracy attainable by any
input-blind policy is bounded by the largest condition share, $1/3$ in our
design, a supremum over the whole policy class that needs no model runs
(\cref{thm:blind}).
\item \textbf{A complete identifying set, run on three models.} We run all
three perturbations the theorem requires, rather than one and inferring, and
on three models spanning two families and two capability tiers. All three
score at or near zero on calibration grounding, far below their own blind
bound rather than above it, because each defaults to one of a small set of
round calibration values that never matches what the scale asserts in any
condition. For the one model with all three relations run, the two
relabeling magnitudes imply opposite-sign exponents, so no single value fits
and the policy falls outside the family \cref{prop:cauchy} covers;
none of the three is separable from the same model instructed to emit a
constant.
\item \textbf{A demonstration that the design matrix has to be right.} The
temporal perturbation's ideal response is $\beta$, not $\beta^{-1}$ as this
protocol previously stated, a sign error in one row that changes the verdict
on that input alone. We release \fn{sgf\_metrics.py}, which builds the
design matrix and computes both bounds for any design of this shape.
\end{enumerate}

%%%%%%%%%%%%%%%%%%%%%%%%%%%%%%%%%%%%%%%%%%%%%%%%%%%%%%%%%%%%%%%%%%%%%%%%%%%%%%%%
\section{Related work and positioning}
\label{sec:related}

Three lines of work bear on this. Probes that perturb a physical prior and
score the change; probes that perturb an input and score whether the stated
reasoning follows; and the blind-baseline and construct-validity literature
that asks what a score should be read against. Each supplies a piece of the
design used here, and none states what its own perturbation identifies.

\paragraph{Counterfactual scaling and instrument reading}
The closest work is \citet{quantiphy2025}, whose benchmark supplies one
physical property in world units and scores size, velocity, and acceleration
against numerical ground truth, then multiplies the supplied prior by
$\alpha\in\{0.001,\dots,700\}$ and finds the prediction does not track
$\alpha y$. We take this as established. Two things differ here: the
calibration is delivered visually, inside the frame rather than as a number in
the prompt, which separates a perceptual failure from a prior-trust failure;
and the score is a ratio of two model outputs rather than an accuracy against
$\alpha y$, which removes the need for ground truth. Neither work states what
its perturbation identifies, which is the gap \cref{sec:theory}
addresses. \citet{measurebench2026} benchmark VLM reading of gauges, dials, and
rulers over 2{,}442 pairs, where the best model reaches $30.2\%$ on real-world
images while exceeding $90\%$ on unit recognition; that work scores readings in
isolation on static images, whereas we measure the same step inside a
downstream video task. \citet{iris2026} and \citet{physicsmind2026} recover
physical parameters and evaluate law-consistent reasoning, and neither perturbs
the calibration.

\paragraph{Faithfulness by perturbation, and its limits}
Faithfulness is probed by perturbing inputs and asking whether the stated
reasoning changes
\citep{lanham2023faithfulness,turpin2023language,medfaith2025,rlvlm2026,stepground2026},
and \citet{crossmodal2026} show VLMs often answer visual queries from text
alone. These probes score a model's \emph{explanation} using raters or
heuristics and report perturbation sensitivity without an insensitive-policy
baseline. The closest statement of our first result is qualitative:
\citet{seth2026assurance} argue that behavioral assurance cannot support the
absence claims governance asks of it, because a model concealing a property
produces the same record as one that never had it. \Cref{thm:rank} makes
that precise and constructive: it says exactly which policies share a record
under a given design, and which perturbation to add to break the tie.
\citet{srivastava2026limits} prove a complementary impossibility from latent
context conditioning; our obstruction needs no hidden trigger and no
distribution shift, and it is removable by construction.

\paragraph{Blind baselines and evaluation validity}
That multimodal benchmarks can be answered without looking is well documented
\citep{blindvlm2024,crossmodal2026,lan2026seeing,singla2026guess,yuan2026mmgist},
and the failure generalizes past multimodal work
\citep{mi2024blindbaselines,anand2018blindfold,pacchiardi2024clevrhans}
(\cref{app:probes} details each). The standard diagnostic in all of these is an
empirical ablation that withholds the input and re-measures one model.
\Cref{thm:blind} is a different object: an analytic supremum over all
blind policies, needing no model runs. An ablation says what one model does
when blinded; the bound says what the best blind policy could score.
\citet{bean2025construct} review 445 LLM benchmarks with 29 expert reviewers
and report that the phenomena measured are often defined loosely and that
statistical testing is rarely performed; \citet{alaa2025medical} and
\citet{salaudeen2025validity} argue similarly. Two requirements from that
literature bear here: a measure should discriminate its target from nearby
properties, and a score should be read against what a system lacking the
property would obtain. \Cref{sec:theory} supplies both in closed
form; \cref{app:probes} tabulates what each probe perturbs and against what
baseline.

%%%%%%%%%%%%%%%%%%%%%%%%%%%%%%%%%%%%%%%%%%%%%%%%%%%%%%%%%%%%%%%%%%%%%%%%%%%%%%%%
\section{What a perturbation probe can identify}
\label{sec:theory}

An evaluation perturbs inputs and reads the induced change in the output. We
ask in general how many perturbations are needed before the record determines
which inputs the system used, and which perturbations will do. The answer is a
rank condition. Everything in this section is used by
\cref{sec:results}; supporting material is in the appendix.

A system receives $n$ positive scalar inputs
$v=(v_1,\dots,v_n)$ and reports a positive quantity $\pi(v)$. In our task
$n=3$: the object's pixel-space displacement $v_1=m$, the calibration $v_2=c$
that the in-frame scale asserts, and the frame rate $v_3=\phi$ that the prompt
asserts. A \emph{relation} with vector $a\in\R_{>0}^{n}$ maps
$v\mapsto(a_1v_1,\dots,a_nv_n)$, and the observable is the response ratio
$r=\pi(av)/\pi(v)$.

\begin{definition}[Log-linear policies]
\label{def:loglin}
$\pi$ is \emph{log-linear with exponent} $\bt\in\R^{n}$ if
$\pi(v)=\kappa\prod_i v_i^{\theta_i}$ for some $\kappa>0$. Write $\Pi$ for this
family. The coordinate $\theta_i$ records how much the report depends on input
$i$, and $\theta_i=0$ means the report ignores it.
\end{definition}

\begin{proposition}[Scope]
\label{prop:cauchy}
Let $\pi>0$ be measurable. The ratio $\pi(av)/\pi(v)$ is independent of $v$ for
every $a$ if and only if $\pi\in\Pi$, in which case
\begin{equation}
\label{eq:loglin}
\log r=\langle\bt,\,\bu\rangle,\qquad \bu:=\log a .
\end{equation}
\end{proposition}

\Cref{prop:cauchy} follows from Cauchy's functional
equation (\cref{app:scope}). It quantifies over the whole positive orthant,
whereas
an experiment probes finitely many operating points, so it says what the
observable can mean rather than describing a test we run. It fixes the scope of
everything below: a response ratio describes the policy rather than the item
mix exactly on $\Pi$, so $\Pi$ is not a convenience but the family the
observable can speak about. Outside $\Pi$ the
ratio moves with the operating point, which the released code detects by
evaluating it at several operating points; \cref{app:loglinscope} gives the two
ways the
restriction bites and reports the temporal relation in full.

\Cref{def:loglin} and \cref{eq:loglin} are the engine. Apply relations
$a_1,\dots,a_k$,
stack their logarithms as the rows of the \emph{design matrix}
$U\in\R^{k\times n}$, and the entire behavioral record is
$\log\mathbf{r}=U\bt$: a linear measurement of the exponent.

\begin{theorem}[Identifiability of input use]
\label{thm:rank}
Let $\pi,\pi'\in\Pi$ have exponents $\bt,\bt'$, and let $U$ be the design
matrix of any finite set of relations. Then $\pi$ and $\pi'$ produce the same
record if and only if $\bt-\bt'\in\ker U$. Consequently:
\emph{(i)} the record determines $\bt$ if and only if
$\operatorname{rank}U=n$, so at least $n$ relations are necessary and any $n$
with linearly independent log-perturbation vectors suffice; \emph{(ii)}
otherwise the policies sharing a record with $\bt$ are exactly the coset
$\bt+\ker U$, of dimension $n-\operatorname{rank}U$; and \emph{(iii)} writing
$e_i$ for the $i$-th standard basis vector, the record determines the single
coordinate $\theta_i$, and equally decides whether $\theta_i=0$, if and only if
$e_i\perp\ker U$, that is if and only if $e_i$ lies in the row space of $U$.
\end{theorem}

\begin{proofsketch}
By \cref{eq:loglin} the record is $U\bt$, so records coincide exactly on
cosets of $\ker U$, giving (ii), and the map is injective exactly when
$\ker U=\{0\}$, giving (i). For (iii), $\theta_i=\langle e_i,\bt\rangle$ is
constant on those cosets iff $\langle e_i,z\rangle=0$ for every $z\in\ker U$;
when it is not constant it sweeps all of $\R$, so the coset contains
both $\theta_i=0$ and $\theta_i\neq0$ policies and even the weaker question is
undecidable. Full proof in \cref{app:rank}.
\end{proofsketch}

\begin{corollary}[Design rule]
\label{cor:mtpair}
A single relation gives $\operatorname{rank}U\le1$, so for $n\ge2$ every
confusion set is at least $(n-1)$-dimensional: one perturbation is never
enough. A set of relations certifies that a system uses all $n$ of its inputs
if and only if their log-perturbation vectors span $\R^{n}$. Enlarging
a relation already in the span cannot help however large its magnitude, so
sharpening a probe never substitutes for adding an independent one.
\end{corollary}

\Cref{fig:teaser}\textbf{b} draws the two one-relation kernels
and the different pair of policies each of them merges. The rule is
indifferent to \emph{why}
$\theta_i=0$: a system that ignores input $i$ incidentally and one trained to
ignore it have the same exponent and the same record, so no enlargement of a
single relation separates incidental from strategic insensitivity
\citep{sandbagging2024}.

\subsection{The blind-policy bound}
\label{sec:blindbound}

Identifiability says what a design can distinguish. It does not say what score
a system that reads nothing would obtain, and that score is not zero.
Calibration grounding elicits the model's read of the scale, $\hat c$, and
scores $\cg=\acc@\tau(\hat c,c^\star)$ against the true asserted value; a
policy that never inspects the ruler still scores above zero whenever its fixed
answer happens to match.

\begin{theorem}[Supremum over calibration-blind policies]
\label{thm:blind}
Let the design present conditions $j=1,\dots,k$ with item shares $p_j$ summing
to $1$, condition $j$ having true asserted calibration $\alpha_j c$. Assume
\emph{(a)} the conditions are $\tau$-separated,
$|\alpha_j-\alpha_{j'}|/\max(\alpha_j,\alpha_{j'})>2\tau$ for all $j\neq j'$,
and \emph{(b)} the condition label is assigned independently of everything else
the policy sees, which a fully crossed design satisfies. Then every policy whose
report is independent of the asserted calibration, deterministic or randomized
and not necessarily in $\Pi$, satisfies $\E[\cg]\le\max_j p_j$, and the
bound is attained by reporting $\alpha_{j^\star}c$ with
$j^\star=\arg\max_j p_j$.
\end{theorem}

\begin{proofsketch}
By \emph{(b)} the law of $\hat c$ does not depend on $j$, so conditioning on a
realized value leaves the condition distributed as $p$. By the triangle
inequality, \emph{(a)} forbids any realized value from falling inside two
tolerance bands, so each realization is correct on at most one condition and
contributes at most $\max_j p_j$; a randomized policy is a mixture and cannot
exceed the maximum of its components. \Cref{app:blindproof} gives the full
argument and
shows neither hypothesis can be dropped.
\end{proofsketch}

This is the one result here that assumes nothing about $\Pi$, which
is what lets it bound the very policies \cref{prop:cauchy} excludes.
Hypothesis \emph{(b)} is not decorative: if the condition label were inferable
from the rest of the input, a blind policy could read it off and exceed the
bound. Our design uses $k=3$ conditions ($\alpha\in\{1,\tfrac12,2\}$) in equal
shares
at $\tau=0.1$, whose pairwise separations are $0.50$, $0.50$ and $0.75$ against
a requirement of $2\tau=0.2$, so the blind bound is exactly $1/3$. Being a
supremum, falling below it means scoring under the \emph{best} blind strategy
rather than under every one. The bound also depends on the condition set
actually scored: dropping the base condition makes the design $k=2$ and the
bound $1/2$; and it scopes what a null can claim, since a one-sided exact
binomial test of $\cg$ against it reaches $80\%$ power only for a true
grounding rate of $0.411$ or more at $n=243$, and $0.468$ to $0.509$ once
elicitations are clustered within the $27$ clips (\cref{app:blindproof}).

%%%%%%%%%%%%%%%%%%%%%%%%%%%%%%%%%%%%%%%%%%%%%%%%%%%%%%%%%%%%%%%%%%%%%%%%%%%%%%%%
\section{Protocol, dataset, and design}
\label{sec:method}

Testing the two results needs a task where the grounded exponent
$\bt^\star$ is known before any model is run and the target is fixed by
dimensional analysis rather than by a rater. A droplet falling past a printed
ruler supplies both, and it makes all three relations physically realizable on
the same clip.

\paragraph{Setup}
A VLM sees a video of a physical event with an in-frame scale reference of
known true calibration $c^\star$ and is asked to report a quantity $q$; a
classical computer-vision pipeline \citep{horn1981optical} supplies $q^\star$
from the same video. Tolerance-band accuracy
$\acc@\tau=\Pr[|\hat q-q^\star|/|q^\star|\le\tau]$ mirrors existing benchmarks
and is not our contribution.

A relation's ideal grounded ratio $\rho$ follows from
$\bt^\star$: the observable to compare against it is $\log r$ itself, which by
\cref{prop:cauchy} is the linear functional $\langle\bt,\bu\rangle$
the relation measures. For $R_\alpha$ and $T_\beta$ that functional is a single
coordinate of $\bt$; for $S_s$ it is the difference $\theta_1-\theta_2$, which
is why the resampling response is read against a target of $1$ rather than
against a coordinate. We report $\log r$ directly rather than a bounded
rescaling of it, so
that \cref{thm:rank} applies to the reported quantity without a change
of variable; \cref{app:sfs} defines the bounded summary used in the supporting
analyses and derives its own blind-policy bound.

\begin{definition}[The three relations]
\label{def:cond}
\emph{Relabeling} $R_\alpha$ presents a scale asserting $\alpha c$ with the
pixel content unchanged. \emph{Resampling} $S_s$ rescales the frame by
$s\neq1$, object and ruler together, with the ruler's asserted physical length
held fixed. \emph{Temporal relabeling} $T_\beta$ asserts a frame rate $\beta$
times the true one with the clip and the spatial calibration unchanged. Their
log-perturbation vectors are $(0,\log\alpha,0)$, $(\log s,-\log s,0)$ and
$(0,0,\log\beta)$, so $\det U=-\log\alpha\,\log s\,\log\beta\neq0$ whenever
$\alpha,s,\beta\neq1$ and the three identify $\bt$ by
\cref{thm:rank}(i); any two of them do not. The grounded exponent for a
reported velocity is $\bt^\star=(1,1,1)$, so the ideal responses are $\alpha$,
$1$ and $\beta$. The $\beta$ target corrects an earlier statement of this
protocol that gave $\beta^{-1}$: a slower asserted rate places more elapsed
time in the denominator of distance over time, so the inferred velocity falls
rather than rises.

\emph{Policy classes.} Within $\Pi$, $\Grd$ is $\bt=\bt^\star$; $\Bld$,
calibration-blind, is $\theta_2=0$; and $\Cst\subset\Bld$, constant, is
$\bt=0$. Each is a linear condition on $\bt$, so
\cref{thm:rank}(iii) decides which of them a given design can test.
\end{definition}

\Cref{fig:setup} shows the three relations instantiated on one clip.

\begin{figure}[t]
\centering
\kilgraphics{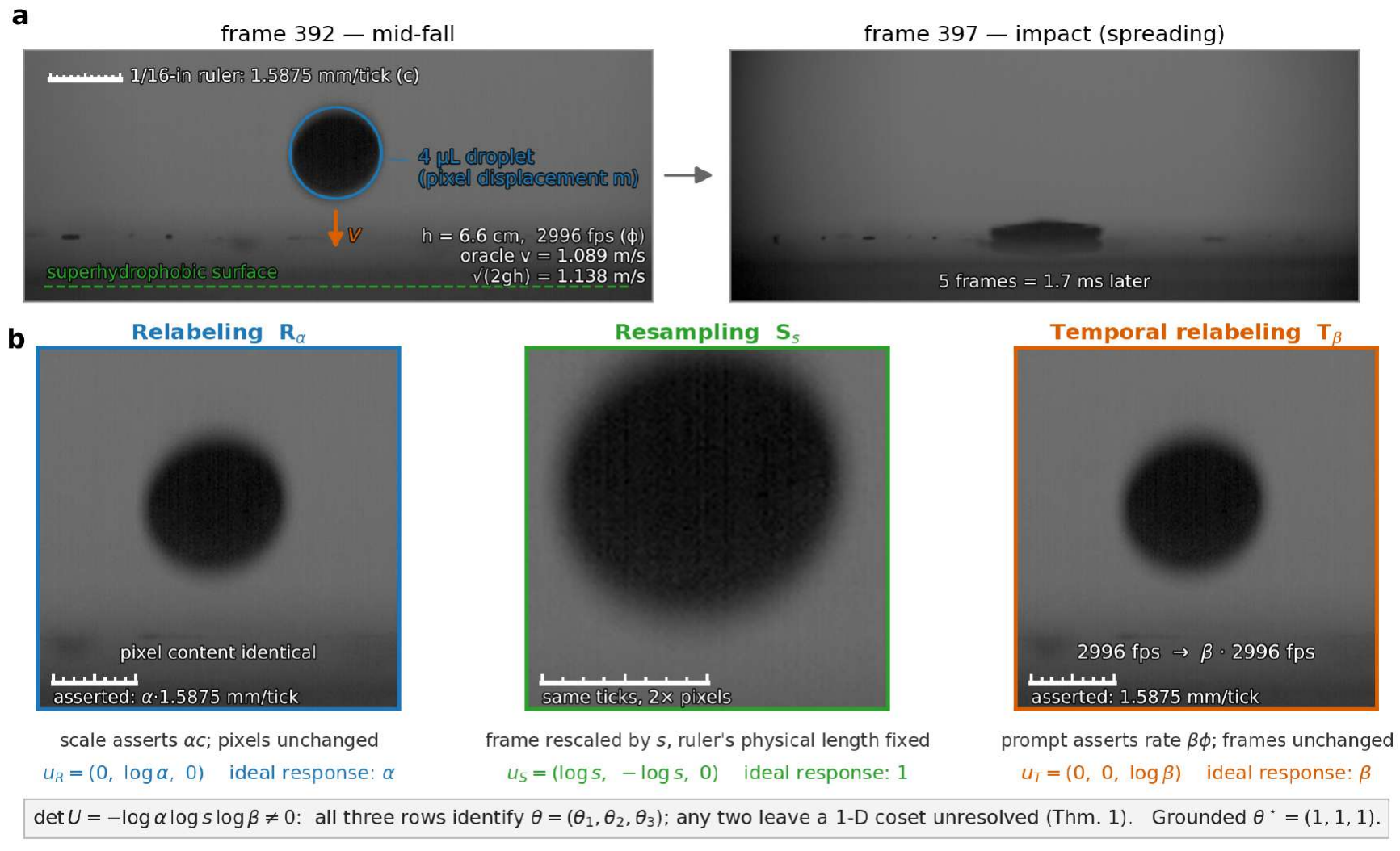}
\caption{\textbf{The identifying design on one clip.}
\textbf{a}, The clip that instantiates $m$, $c$, and $\phi$ (mid-fall and
impact frames; \cref{app:dataset} has full capture detail). \textbf{b}, The
three
relations of \cref{def:cond} applied to it, with each one's
log-perturbation vector and grounded target below its panel.}
\label{fig:setup}
\end{figure}

\paragraph{Dataset and oracle}
The dataset is the instantiation, not the contribution; it exists to make the
two relations physically realizable with an exact oracle. The recordings are
high-speed video ($2996$\,fps) of $4\,\mu$L droplets impacting a
superhydrophobic surface, released from a fixed height ($6.6$\,cm) outside the
field of view, in five sessions of which three survived verification (two
were excluded, for a corrupted calibration clip and for absent provenance).
Of the three, the one used for model evaluation uses a printed ruler
graduated in $1/16$\,inch ($1.5875$\,mm per tick, confirmed by direct pixel
measurement of its alternating long and short marks after an earlier version
of this protocol misidentified it as $1$\,mm graduated); a second, excluded
from model evaluation, uses a pin-comb target at the same pitch; a third,
a genuine $1$\,mm ruler, was lost with the compute environment it was staged
on. Calibrations were measured per session in pixels per tick and converted
using each session's actual graduation, not a shared assumed constant;
\cref{app:dataset} gives the full correction record, including an independent
free-fall cross-check ($v=\sqrt{2gh}$) whose mean residual fell from $40\%$
under the wrong constant to $6.7\%$ corrected. The oracle returned a valid
pre-impact velocity for 48 of 62 clips attempted; model evaluation uses
every high-confidence clip in the surviving ruler session, $27$ rather than
the original $12$-clip two-session subset --- larger, and it improves the
power problem above, but it removes the paper's only cross-session
evidence, so within-model claims concern one physical setup rather than
sessions in general. Full capture details are in \cref{app:dataset}.

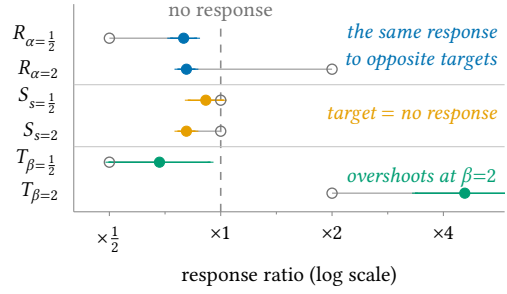
\begin{figure}[t]
\centering
\begin{tikzpicture}[font=\scriptsize]
\pgfplotsset{
  every axis/.append style={
    scale only axis, axis lines*=left,
    axis line style={draw=black!45,line width=0.4pt},
    tick style={draw=black!45,line width=0.4pt}, tick align=outside,
    major tick length=1.6pt, minor tick length=1.0pt,
    label style={font=\scriptsize}, tick label style={font=\scriptsize},
    x label style={yshift=2pt}}}

%% ---------------- every response, against its own two targets ----------------
\begin{axis}[name=B,
  width=5.7cm, height=2.75cm,
  xmin=-1.32, xmax=2.55, ymin=0.4, ymax=7.1,
  xtick={-1,0,1,2}, xticklabels={$\times\tfrac12$,$\times1$,$\times2$,$\times4$},
  minor xtick={-0.5,0.5,1.5},
  ytick={1,2,3,4,5,6},
  yticklabels={$T_{\beta=2}$,$T_{\beta=\frac12}$,$S_{s=2}$,$S_{s=\frac12}$,$R_{\alpha=2}$,$R_{\alpha=\frac12}$},
  ytick style={draw=none}, y tick label style={align=right},
  xlabel={response ratio (log scale)}]
  \draw[black!45,dashed,line width=0.6pt] (axis cs:0,0.4) -- (axis cs:0,6.45);
  \node[anchor=south,inner sep=1pt,text=black!55] at (axis cs:0,6.45) {no response};
  \draw[black!22,line width=0.4pt] (axis cs:-1.32,4.5) -- (axis cs:2.55,4.5);
  \draw[black!22,line width=0.4pt] (axis cs:-1.32,2.5) -- (axis cs:2.55,2.5);
  % gap from target to measured
  \draw[black!35,line width=0.5pt] (axis cs:-1,6) -- (axis cs:-0.333,6);
  \draw[black!35,line width=0.5pt] (axis cs:1,5)  -- (axis cs:-0.308,5);
  \draw[black!35,line width=0.5pt] (axis cs:0,4)  -- (axis cs:-0.134,4);
  \draw[black!35,line width=0.5pt] (axis cs:0,3)  -- (axis cs:-0.308,3);
  \draw[black!35,line width=0.5pt] (axis cs:-1,2) -- (axis cs:-0.550,2);
  \draw[black!35,line width=0.5pt] (axis cs:1,1)  -- (axis cs:2.192,1);
  % grounded targets
  \addplot[only marks,mark=o,mark size=1.9pt,black!55,line width=0.5pt]
    coordinates {(-1,6)(1,5)(0,4)(0,3)(-1,2)(1,1)};
  % measured, with cluster-bootstrap 95% intervals (27 clips, 10000 draws)
  \addplot[only marks,mark=*,mark size=1.9pt,wBlue,
    error bars/.cd,x dir=both,x explicit,error bar style={wBlue,line width=0.6pt},
    error mark options={mark size=1.1pt}]
    coordinates {(-0.333,6) +- (0.121,0.089) (-0.308,5) +- (0.081,0.075)};
  \addplot[only marks,mark=*,mark size=1.9pt,wOrange,
    error bars/.cd,x dir=both,x explicit,error bar style={wOrange,line width=0.6pt},
    error mark options={mark size=1.1pt}]
    coordinates {(-0.134,4) +- (0.162,0.183) (-0.308,3) +- (0.081,0.075)};
  \addplot[only marks,mark=*,mark size=1.9pt,wGreen,
    error bars/.cd,x dir=both,x explicit,error bar style={wGreen,line width=0.6pt},
    error mark options={mark size=1.1pt}]
    coordinates {(-0.550,2) +- (0.459,0.380) (2.192,1) +- (0.447,0.369)};
  \node[anchor=east,inner sep=2pt,text=wBlue,font=\scriptsize\itshape,align=right]
    at (axis cs:2.53,5.72) {the same response\\to opposite targets};
  \node[anchor=east,inner sep=2pt,text=wOrange,font=\scriptsize\itshape]
    at (axis cs:2.53,3.52) {target $=$ no response};
  \node[anchor=east,inner sep=2pt,text=wGreen,font=\scriptsize\itshape]
    at (axis cs:2.53,1.52) {overshoots at $\beta{=}2$};
\end{axis}
\end{tikzpicture}
\caption{\textbf{Where the model actually responds}
(Qwen2.5-VL-7B; \cref{tab:e2} has $\cg$ and the other two models' response
ratios). Mean response ratio per relation and magnitude, with $95\%$
cluster-bootstrap intervals, against the grounded target
(\cref{def:cond}, open) and the no-response value (dashed). The
two relabeling magnitudes have targets on opposite sides yet return nearly
the same ratio, so no single exponent fits; the resampling target coincides
with no response, so that relation cannot identify alone; the temporal
relation moves past its target rather than toward it at $\beta{=}2$.}
\label{fig:cg}
\end{figure}

\paragraph{Models and conditions}
We evaluate three models under identical prompts across three paraphrased
seeds on the 27-clip design: Qwen2.5-VL-7B-Instruct \citep{bai2025qwen25vl}
(open-weight, local, prior-generation), Gemini 3.1 Pro Preview (OpenRouter
API, current frontier), and Qwen3-VL-30B-A3B-Instruct (open-weight, local),
chosen to test whether the small-model result is specific to Qwen2.5-VL-7B's
size and generation. Both local models used greedy decoding at pinned
checkpoint revisions; Gemini was queried at temperature $0$ at a
snapshot-resolved OpenRouter endpoint, which cannot offer the guarantee a
checkpoint hash gives. \Cref{app:provenance} lists the revisions, endpoint
slug, and
\$20.85 total API spend.
\textbf{Sanity check}: feed the metric
the oracle value times $\alpha$ and a fixed $1$\,m/s emitter. \textbf{Main
evaluation}: for each model and $\alpha\in\{0.5,2\}$, record $\log r_R$,
$\acc@10\%$
and $\cg$ over 27 clips $\times$ 3 conditions $\times$ 3 seeds, 243 calls per
model. \textbf{Resampling control}: $S_s$, $s\in\{0.5,2\}$ on the same
clips, 162 calls per model. \textbf{Instructed-constant control}: 81
calls per model, no seed sweep since the instruction is the manipulation.
The hidden-prior probe, decoupling analysis, and robustness sweep
are in \cref{app:supporting} and carry no claim in what follows.

\begin{table}[t]
\centering
\caption{\textbf{All three models score far below their own blind bound on
calibration grounding; only Qwen2.5-VL's resampling pair resolves cleanly to
$\Cst$.} Each model defaults to one of a small set of round calibration
values (Qwen2.5-VL splits between $0.5$ and $1.0$\,mm/tick; Gemini and
Qwen3-VL report $1.0$ on nearly every item) that never matches the true
$1.5875$\,mm/tick the scale asserts in any condition, so $\cg$ sits at or
near zero uniformly rather than tracking the blind bound. The
instructed-constant control shows the design cannot separate any of the
three from a policy explicitly told to ignore the video. $\cg$, calibration
grounding; $r_R,r_S,r_T$, response ratios under relabeling, resampling, and
temporal relabeling.}
\label{tab:e2}
\small
\setlength{\tabcolsep}{4pt}
\renewcommand{\arraystretch}{1.12}
\resizebox{\linewidth}{!}{%
\begin{tabular}{@{}l r r r r@{}}
\toprule
\kilth{Quantity} & \kilth{Blind/constant bound} & \kilth{Qwen2.5-VL-7B} & \kilth{Gemini 3.1 Pro} & \kilth{Qwen3-VL-30B-A3B} \\
\midrule
\multicolumn{5}{@{}l}{\emph{Main evaluation (relabeling $R_\alpha$), $n=243$; $240^{h}$; $243$}}\\
\quad $\cg$                          & $0.333$ exact (\cref{thm:blind}) & $0.000$ $[0.000,0.016]^{a}$ & $0.004$ $[0.001,0.023]^{a}$ & $0.000$ $[0.000,0.016]^{a}$ \\
\addlinespace[2pt]
\multicolumn{5}{@{}l}{\emph{Resampling control ($S_s$), $n=81$ each}}\\
\quad $r_R$ (mean; $\alpha=\tfrac{1}{2}$, $2$) & ideal grounded $0.5$, $2.0$ & $0.794$, $0.808$ & $1.610$, $1.653$ & $1.059$, $1.048$ \\
\quad $r_S$ (mean; $s=\tfrac{1}{2}$, $2$)      & ideal grounded $1.0$, $1.0$ & $0.911$, $0.808$ & $1.474$, $1.021$ & $0.939$, $1.588$ \\
\quad $r_T$ (mean; $\beta=\tfrac{1}{2}$, $2$)  & ideal grounded $0.5$, $2.0$ & $0.683$, $4.569$ & not run & not run \\
\quad classification$^{g}$ & --- & $\Cst$ & unclassified$^{i}$ & mixed$^{j}$ \\
\addlinespace[2pt]
\multicolumn{5}{@{}l}{\emph{Instructed-constant control, $n=81$ each}}\\
\quad $\cg$          & $0.333$ exact & $0.333$ $[0.240,0.441]^{a}$ & $0.333$ $[0.240,0.441]^{a}$ & $0.333$ $[0.240,0.441]^{a}$ \\
\quad compliance     & --- & $81/81$ & $81/81$ & $81/81$ \\
\addlinespace[2pt]
\multicolumn{5}{@{}l}{\emph{Reference: task accuracy, reported for context and not a grounding claim}}\\
\quad $\acc@10\%$ ($n=81$; $80$; $81$) & --- & $0.000$ $[0.000,0.045]^{c}$ & $0.100$ $[0.052,0.183]^{c}$ & $0.000$ $[0.000,0.045]^{c}$ \\
\quad $\acc@30\%$                      & --- & $0.074$ & $0.313$ & $0.037$ \\
\bottomrule
\end{tabular}}
\tabnote{$^{a}$Wilson interval.
$^{c}$Clopper--Pearson exact.
$^{g}$Per \cref{thm:rank}: $(\alpha,1)$ for $\Grd$, $(1,{\neq}1)$
for $\Bld\setminus\Cst$, $(1,1)$ for $\Cst$; descriptive outside $\Pi$.
$^{h}$Gemini's $n$ is 240: 3 items failed to parse.
$^{i,j}$Per-item classification breakdowns in \cref{app:supporting}.}
\end{table}

%%%%%%%%%%%%%%%%%%%%%%%%%%%%%%%%%%%%%%%%%%%%%%%%%%%%%%%%%%%%%%%%%%%%%%%%%%%%%%%%
\section{Results}
\label{sec:results}

Every quantity below is read against a null the previous section fixes:
$\cg$ against the blind bound of \cref{thm:blind}, and each response
ratio against both its grounded target and the value a policy ignoring that
input would give. The sanity check establishes neither: the constructed faithful
re-measurer returns the ideal ratio and the fixed-prior emitter returns $1$,
but the oracle cancels in that ratio, so the result holds by arithmetic for
\emph{any} oracle and checks the implementation rather than the oracle
(\cref{app:supporting}).

\paragraph{A score far below its own blind bound, from a policy that reads nothing}
\Cref{tab:e2} reports the main results, and the headline is uniform
failure rather than a ranking. All three models score at or near zero on
$\cg$, well under the blind bound of $1/3$: Qwen2.5-VL $0/243$, Gemini
$1/240$, Qwen3-VL $0/243$. \textbf{We do not read this as evidence against
grounding either}; a design that cannot separate a calibration-blind policy
from a constant one (\cref{cor:mtpair}) cannot conclude much from
either a high or a low score. The per-condition breakdown shows why the score
is uniformly low rather than tracking the condition mix: each model reports
one of a small set of fixed calibrations regardless of what is asserted ---
Qwen2.5-VL splits between $0.5$ and $1.0$\,mm/tick ($214$ and $27$ of $243$
items), Gemini and Qwen3-VL report $1.0$ on nearly every item ($236/240$ and
$243/243$) --- and none of these defaults falls within $\tau$ of the true
$1.5875$\,mm/tick asserted in any condition, which is what
\cref{cor:mtpair} says a relabeling-only design cannot rule out
either way. Qwen's reported velocity shows the same collapse, $1.5$\,m/s on
$214$ of $243$ items, with tolerance-band accuracy $0/81$ at $\acc@10\%$ and
$6/81$ at $\acc@30\%$; three of Gemini's 243 items ($1.2\%$) failed to parse
and are excluded. By \cref{cor:mtpair} the relabeling coordinate
alone cannot separate calibration-blind from constant for any of the three,
and the corrected numbers make that indistinguishability the empirical
outcome for all three rather than two of three.

\paragraph{The resampling control resolves the spatial coordinates directly}
Here $r_R\approx r_S\approx0.8$ for Qwen2.5-VL: both close to $1$ and close
to each other, the $(1,1)$ signature of $\Cst$, not the $(1,{\neq}1)$
signature of $\Bld\setminus\Cst$. This policy lies outside $\Pi$ (below), so
we read that as a description of its aggregate behavior rather than an
identification under \cref{thm:rank}, which quantifies over $\Pi$ ---
the first point where the classification is a direct empirical reading
rather than a bound the model might or might not clear. For Gemini the same
pair does not resolve, its per-item $r_R$ ranging up to $5.5$, so noise
rather than a competing signature prevents classification. For Qwen3-VL the
aggregate ratios move away from $1$, ruling out pure $\Cst$; per-item
classification splits near-evenly at $\alpha/s{=}0.5$ and gives a
\emph{calibration-blind} majority at $\alpha/s{=}2$ (\cref{tab:e2},
notes $i,j$).

\paragraph{The instructed-constant control shows the design cannot tell incidental from strategic insensitivity}
A system prompt instructing the model to report $1.5$\,m/s and the session's
own true base-condition calibration regardless of video content produced
\emph{exact} compliance for all three models: $81/81$ items returned the
instructed values, giving $\cg=27/81=0.333$, exactly the blind bound by
construction and now \emph{higher} than any uninstructed model's (Fisher
exact $p<10^{-17}$ against Qwen's $0/243$), simply because the instructed
value matches the design's own base condition. That sharpens rather than
weakens the point: the design still cannot separate this deliberately
insensitive policy from a model that happens to read the scale correctly
only on the base condition. We do not claim any model is sandbagging; a
change-expecting relation cannot settle it either way.

\paragraph{Where the response sits, and why the temporal axis cannot be summarized by one exponent either}
With all three relations run for Qwen2.5-VL, \cref{def:cond} gives a
full-rank design, so $\bt$ is identifiable in principle. It is not recoverable
in practice here, and the way it fails is itself the result. Solving
$\log r=\theta_2\log\alpha$ separately at the two relabeling magnitudes
(\cref{fig:cg}) gives $\theta_2=+0.33$ and $\theta_2=-0.31$:
opposite signs, near zero, which is what a near-constant output with
condition-wise drift produces and a grounded policy ($\theta_2=1$) does not.
No single exponent fits both, so this policy lies outside $\Pi$ and, by
\cref{prop:cauchy}, its aggregate response ratio is not a property
of the policy alone. Resampling behaves likewise, moving the report toward
neither target. The temporal axis moves furthest of the three, but not
consistently toward the target: the mean realized ratio overshoots at
$\beta{=}2$ ($4.569$ against target $2.0$) while the median undershoots at
$\beta{=}0.5$ and sits near the blind value at $\beta{=}2$, the same
mean/median instability documented below for Gemini's spatial ratios, so the
temporal coordinate's direction cannot be pinned to one summary statistic.
What is stable is only that its magnitude moves further than either spatial
coordinate's, and it has no exponent to be a component of, since the
relabeling inconsistency above already places this policy outside $\Pi$.
\Cref{app:loglinscope} gives all six response ratios with intervals; the
corrected
calibration constant changes the temporal ratios specifically because
$T_\beta$'s prompt asserts that value as text, unlike the purely visual
spatial relations.

%%%%%%%%%%%%%%%%%%%%%%%%%%%%%%%%%%%%%%%%%%%%%%%%%%%%%%%%%%%%%%%%%%%%%%%%%%%%%%%%
\section{What this design cannot identify}
\label{sec:threats}

An evaluation used to license a decision should state what it cannot
license, starting with the design's own completeness: only Qwen2.5-VL has
all three relations run, so only there is $U$ full rank.
For the other two, $\operatorname{rank}U=2<3$, and were they in $\Pi$
\cref{thm:rank}(ii) would confine the claim to the coset
$\bt+\operatorname{span}\{e_3\}$: their temporal axis is unperturbed, so no
statement about it is available either way. All of this rests on the single
surviving session.

\paragraph{A text-calibration positive control: the failure is not only perceptual}
The instructed-constant control supplies the negative control, and $T_\beta$
shows the same model does move
on some axis, which is what makes the spatial nulls informative rather than
uniform. But no positive control existed on the spatial axes themselves before
this experiment: no evaluated model had produced a spatial response
distinguishable from no response, so the design had been shown to detect the
\emph{absence} of spatial grounding and not its presence. We supplied the true
calibration as text instead of a scale-reference video, on the same 27 clips,
removing the perceptual step while leaving the pixel-to-physical conversion
step intact, and ran both $R_\alpha$ and $S_s$ on the result (243 and 162
calls per model). We recorded the prediction in advance: \cref{thm:rank}
predicts $(\alpha,1)$ for a policy that converts, against the $(1,1)$ measured
uninstructed.

\textbf{No model produced the predicted signature.} All three classify
\texttt{unclassified} at both magnitudes under mean and median alike, with
one near-miss: Qwen3-VL-30B-A3B's median relabeling ratio at $\alpha{=}2$ is
$2.00$, exactly the grounded target, but its paired resampling median
($1.51$) is too far from $1.0$ to register as grounded (\cref{app:supporting}
has the
full per-model ratios and zero-response rates, all at or above each model's
video-based rate). \textbf{What this buys the paper.} Removing the
perceptual step did not produce a single grounded response, on three models
spanning two families and two capability tiers, which rules out one
explanation the null result could not previously exclude: the uniform
failure to detect spatial grounding is not solely because these models
cannot read a ruler. At least as much of it sits in the conversion step
itself, or in unfamiliarity with this text-plus-video prompt format, a
distinction the earlier design could not draw and this one does not isolate
either, since that needs a control varying prompt format alone.

\paragraph{Legibility and oracle error}
Unresolvable graduations would make $\cg$ measure legibility rather than
grounding. The temporal-sufficiency half of this control holds ($98.5$--$100\%$
of each clip's frames ingested); the static-legibility half does not:
queried on a still frame at both native and delivered resolution, the model
reported $1$\,mm/tick, a $37\%$ error against the true $1.5875$, consistent
with the same misreading persisting under video (\cref{tab:e2}). Oracle
error is likewise governed by the same correction: across the 41 tracked
clips in the two graduation-corrected sessions, the relative residual
against free fall has median $0.043$ and 90th percentile $0.081$, both under
$\tau=0.10$ (mean $0.067$ for the 27-clip evaluation session alone) --- versus
a mean $0.402$ under the uncorrected constant, which is what first flagged
the error. The $\tau$-sweep, recomputed against the corrected true
velocities in \cref{app:supporting}, resembles the originally hypothesized
decoupled
pattern for none of the three models at any $\tau$. Twenty-seven clips, one
phenomenon, three models, one temporal axis.

\paragraph{What survives these limits}
\Cref{cor:mtpair} and \cref{thm:blind} are statements about
design matrices, not about droplets: any evaluation that perturbs one input
and reads a flat response is in the position described here, the remedy is
a second relation independent of the first rather than a sharper version of
the same one, and the blind bound is computable for any tolerance-scored
design before a model is run. What the instantiation adds is the
demonstration that a score, whether it clears that bound or, as here, sits
far under it, can come from a policy indistinguishable, on every quantity
this design measures, from one instructed to ignore the video.

%%%%%%%%%%%%%%%%%%%%%%%%%%%%%%%%%%%%%%%%%%%%%%%%%%%%%%%%%%%%%%%%%%%%%%%%%%%%%%%%
\section{Conclusion}
\label{sec:conclusion}

A perturbation-based evaluation identifies which inputs a system uses exactly
when its log-perturbation vectors span the input space, so one relation is
never enough and a sharper version of the same relation never substitutes for
an independent one. The score such a test awards a policy that reads nothing
is computable in closed form before any model is run. Running a complete
identifying set on three vision--language models, all three sit far below that
bound and none is separable from the same model instructed to ignore the video.

%%%%%%%%%%%%%%%%%%%%%%%%%%%%%%%%%%%%%%%%%%%%%%%%%%%%%%%%%%%%%%%%%%%%%%%%%%%%%%%%
%% BACK MATTER
%%%%%%%%%%%%%%%%%%%%%%%%%%%%%%%%%%%%%%%%%%%%%%%%%%%%%%%%%%%%%%%%%%%%%%%%%%%%%%%%
\begin{kilrepro}
\fn{sgf\_metrics.py} (numpy only) builds $U$ and computes every quantity
reported here; the renderer, prompts, oracle scripts, inference clients,
per-item outputs, and a reconciliation script that checks every table cell
against the raw outputs are in the Code and Data Supplement
(\cref{app:provenance} has the reconciliation and provenance detail), and the
AAAI-27 reproducibility checklist is submitted separately.
\end{kilrepro}

% NOTE: the source paper's author block is a placeholder and carries no
% author-contributions statement. Left commented rather than invented.
% \begin{contributions}
% \end{contributions}

% NOTE: no acknowledgments section exists in the source paper.
% \begin{acknowledgments}
% \end{acknowledgments}

% NOTE: no funding statement exists in the source paper.
% \begin{funding}
% \end{funding}

\begin{availability}
% This section SURVIVES blind mode (most venues require it), so the URL is
% anonymized explicitly -- \ifbool{KIL@blind}{...}{...} does it automatically.
\fn{sgf\_metrics.py}, the renderer, prompts, oracle scripts, inference clients,
per-item outputs and reconciliation script are available at
\ifbool{KIL@blind}
{\url{https://anonymous.4open.science/r/XXXXXX}}
  {\url{https://github.com/KurbanIntelligenceLab/one-perturbation-is-not-enough}}.
Model provenance, checkpoint revisions and compute are in
\cref{app:provenance}.
\end{availability}

\begin{conflicts}
This protocol is
correspondingly a diagnostic, not a certification: by \cref{thm:rank}
an incomplete design constrains $\bt$ only modulo $\ker U$, so a low score
does not distinguish the policies it confuses and a high score is evidence
only about the span perturbed, and a measure mistaken for a certificate is
worse than none. The data are physical measurements with no human subjects.
AI tools assisted drafting; the authors are responsible for all content.
\end{conflicts}

%%%%%%%%%%%%%%%%%%%%%%%%%%%%%%%%%%%%%%%%%%%%%%%%%%%%%%%%%%%%%%%%%%%%%%%%%%%%%%%%
\bibliography{references}

%%%%%%%%%%%%%%%%%%%%%%%%%%%%%%%%%%%%%%%%%%%%%%%%%%%%%%%%%%%%%%%%%%%%%%%%%%%%%%%%
\appendix

\noindent This appendix supplies the proofs whose main-text statements carry a
pointer here, together with the tightness analyses, the failure modes of the
standing assumption, and the supporting experiments that carry no claim in the
main paper. \Cref{cor:mtpair} follows from \cref{thm:rank} in one line and is
not reproved here; every other result has its full proof below, and every
section below is referenced from the main body.

\section{Proof of the scope proposition}
\label{app:scope}

Let $\pi:\R_{>0}^{n}\to\R_{>0}$ be measurable.

\begin{proof}[Proof of \cref{prop:cauchy}]
\emph{Sufficiency.} If $\pi(v)=\kappa\prod_i v_i^{\theta_i}$ then
$\pi(av)/\pi(v)=\prod_i a_i^{\theta_i}$, independent of $v$, and taking
logarithms gives $\log r=\langle\bt,\log a\rangle$.

\emph{Necessity.} Suppose $\pi(av)/\pi(v)=h(a)$ for all $v$ and all $a$. Put
$g=\log\pi\circ\exp$ componentwise, so $g:\R^{n}\to\R$ is
measurable and $g(x+y)-g(x)=\log h(e^{y})$ for all $x,y$. Fixing $x=0$ gives
$\log h(e^{y})=g(y)-g(0)$, and substituting back yields
$g(x+y)=g(x)+g(y)-g(0)$. Hence $\tilde g:=g-g(0)$ is additive and measurable,
so by the Cauchy--Banach--Steinhaus argument for measurable additive functions
it is linear: $\tilde g(x)=\langle\bt,x\rangle$ for some
$\bt\in\R^{n}$. Exponentiating, $\pi(v)=\kappa\prod_i v_i^{\theta_i}$
with $\kappa=e^{g(0)}$.
\end{proof}

\begin{remark}
Measurability is not decorative: without it the Cauchy equation admits
pathological additive solutions. It is satisfied by any policy realizable as a
deterministic decoding of a finite computation, which covers every system we
evaluate. Outside $\Pi$ the response ratio depends on the operating point, so
an aggregate ratio reports the item mix as much as the policy; the released
code checks this by evaluating the ratio at several operating points.
\end{remark}

\section{Proof of the identifiability theorem}
\label{app:rank}

\begin{proof}[Proof of \cref{thm:rank}]
By \cref{prop:cauchy} the record of $\bt$ under relations with
log-perturbation
vectors $\bu_1,\dots,\bu_k$ is $\log\mathbf{r}=U\bt$, where $U$ has rows
$\bu_j^{\top}$.

\emph{(ii)} $U\bt=U\bt'$ iff $U(\bt-\bt')=0$ iff $\bt-\bt'\in\ker U$, so the
set of exponents producing the same record as $\bt$ is exactly the coset
$\bt+\ker U$, an affine subspace of dimension $n-\operatorname{rank}U$.

\emph{(i)} The record determines $\bt$ iff that coset is a single point, iff
$\ker U=\{0\}$, iff $\operatorname{rank}U=n$. Since $\operatorname{rank}U\le k$,
this forces $k\ge n$; and if $\bu_1,\dots,\bu_n$ are linearly independent then
$\operatorname{rank}U=n$ and the record determines $\bt$.

\emph{(iii)} The quantity $\theta_i=\langle e_i,\bt\rangle$ is determined by the
record iff it is constant on every coset $\bt+\ker U$, iff
$\langle e_i,z\rangle=0$ for all $z\in\ker U$, iff
$e_i\in(\ker U)^{\perp}=\operatorname{row}(U)$. So a design certifies whether
input $i$ is used exactly when $e_i$ lies in the row space of $U$.
\end{proof}

\begin{remark}[Recovering the earlier statements]
For $n=2$ with relabeling alone, $\bu_R=(0,\log\alpha)$ and
$\ker U=\operatorname{span}\{e_1\}$, so the constant exponent $(0,0)$ and any
tracking-but-calibration-blind exponent $(t,0)$ share a record: this is the
non-identifiability of a change-expecting relation. With resampling alone,
$\bu_S=(\log s,-\log s)$ and $\ker U=\operatorname{span}\{(1,1)\}$, which
merges the grounded exponent $(1,1)$ with $(0,0)$. The witness
$f(x)=x^{\theta}$ with $\theta=\log(ab)/\log a$, used in earlier statements of
this result to exhibit a collision for an arbitrary single relation $(a,b)$, is
exactly a translate of the grounded exponent along $\ker U$; the theorem
produces it mechanically for any $n$ and any relation set, and says when no
such witness exists.
\end{remark}

\section{Proof of the blind bound, the power calculation, and tightness}
\label{app:blindproof}

\begin{proof}[Proof of \cref{thm:blind}]
Let the design present conditions $j=1,\dots,k$ with item shares $p_j$ summing
to $1$, condition $j$ having true asserted calibration $\alpha_j c$ with
$\alpha_1=1$, and assume $\tau$-separation:
$|\alpha_j-\alpha_{j'}|/\max(\alpha_j,\alpha_{j'})>2\tau$ for $j\neq j'$.

A calibration-blind policy's reported $\hat c$ has a law that does not depend on
$j$, since $j$ enters the input only through the asserted calibration. Condition
on a realized value $v$. Scoring is $\acc@\tau$, so $v$ is counted correct on
condition $j$ exactly when $|v-\alpha_j c|\le\tau\alpha_j c$.

Suppose $v$ were correct on two conditions $j\neq j'$. Then
\begin{align*}
|\alpha_j-\alpha_{j'}|\,c &\le |v-\alpha_j c| + |v-\alpha_{j'} c|\\
&\le \tau(\alpha_j+\alpha_{j'})c
\le 2\tau\max(\alpha_j,\alpha_{j'})c,
\end{align*}
using the triangle inequality and then $\alpha_j+\alpha_{j'}\le
2\max(\alpha_j,\alpha_{j'})$. Dividing by $\max(\alpha_j,\alpha_{j'})c>0$
contradicts $\tau$-separation. Hence each realized $v$ is correct on at most one
condition, so $\E[\cg\mid \hat c=v]\le\max_j p_j$.

A randomized blind policy is a mixture over such $v$, and
$\E[\cg]=\int \E[\cg\mid \hat c=v]\,d\mu(v)\le\max_j p_j$,
because a mixture cannot exceed the supremum of its components. Randomization
therefore does not help, which is what makes the bound a statement about the
whole class rather than about one strategy.

The bound is attained: reporting the constant $v=\alpha_{j^\star}c$ with
$j^\star=\arg\max_j p_j$ is correct on every item of condition $j^\star$ and so
achieves $\max_j p_j$ exactly. Under balance $p_j=1/k$ and the supremum is
$1/k$.
\end{proof}

\paragraph{Power calculation}
Consider a one-sided exact binomial test of $H_0:\cg=p_0$ against
$H_1:\cg>p_0$ at level $a$ with $n$ independent elicitations. The critical count
is the smallest $k^\star$ with
$\sum_{i=k^\star}^{n}\binom{n}{i}p_0^{i}(1-p_0)^{n-i}\le a$, and the power at an
alternative $p_1$ is
$\beta(p_1)=\sum_{i=k^\star}^{n}\binom{n}{i}p_1^{i}(1-p_1)^{n-i}$.
For the current $n=243$, $p_0=1/3$, $a=0.05$ design, $\beta$ crosses $0.80$ at
$p_1=0.411$. Elicitations are clustered within clips, so replacing $n$ by
$n_{\mathrm{eff}}=\lfloor n/\mathrm{deff}\rfloor$ with
$\mathrm{deff}=1+(\bar m-1)\rho$ and $\bar m=9$ gives, for $\rho=0.2$,
$\mathrm{deff}=2.6$, $n_{\mathrm{eff}}=93$, and a requirement of $0.468$; for
$\rho=0.5$, $\mathrm{deff}=5.0$, $n_{\mathrm{eff}}=48$, and $0.509$. We did not
estimate $\rho$ from our data and give this range to indicate direction and
rough size, not as a measurement. Note that $k^\star$ is an integer, so the
requirement is not monotone in $n_{\mathrm{eff}}$ and must be read at the exact
value used. All figures are reproduced by \fn{min\_detectable\_rate} in the
released code. Earlier statements of this protocol printed $0.50$ and $0.58$ for
the
clustered requirement; those do not reproduce under the method above and are
superseded by the values in this paragraph, both regenerated from
\fn{min\_detectable\_rate}.

\section{Scope of the log-linear model, and the temporal relation}
\label{app:loglinscope}

\Cref{prop:cauchy} restricts the identifiability analysis to log-linear
policies,
the family for which a response ratio describes the policy rather than the item
mix. The restriction bites in two identifiable ways, and \cref{thm:blind} is
stated so
as to cover the excluded cases.

\textbf{Policies outside $\Pi$.} A saturating or additive policy has a response
ratio that moves with the operating point, so its aggregate ratio is not a
property of the policy. \Cref{thm:rank} then constrains only the log-linear
projection of its behavior, and the blind bounds of \cref{thm:blind} and
\cref{prop:sfsbound}, which
make no log-linearity assumption, are what covers the rest. The released code
flags this case by evaluating the ratio at several operating points.

\textbf{Mixed-dimension reports.} If a model reports a quantity that is not
dimensionally homogeneous, for instance a Reynolds or Weber number combining
length, velocity, and material constants, the grounded exponent is not the
all-ones vector and must be recomputed from the quantity's definition;
$\bt^\star$ is a property of the reported quantity, not of the protocol.

\textbf{Temporal coupling.} Relabeling holds the frame rate fixed, so the time
axis is unperturbed and velocity has $d=1$. A model that misestimates the frame
rate introduces an error relabeling cannot reveal, because that error is common
to both conditions and cancels in the ratio. The temporal analogue $T_\beta$
relabels the asserted frame rate instead: since there is no in-frame
``frame-rate ruler,'' the asserted rate is given directly in the prompt
alongside the true spatial calibration, and only the claimed playback rate
changes, by $\beta\in\{0.5,2\}$ relative to the true $15$\,fps ingestion rate.
A faithful measurer's velocity must then scale by $\beta$, giving
$\theta_3=+1$: a slower asserted rate places more elapsed time in the
denominator of distance over time, so the inferred velocity falls. Earlier
statements of this protocol gave $\beta^{-1}$, which is a sign error.

\textbf{All six response ratios for Qwen2.5-VL, with intervals.} Relabeling
moves the report by $0.794$ at $\alpha{=}0.5$ (log$_2$ mean $-0.332$, cluster
$95\%$ $[-0.453,-0.243]$) and $0.808$ at $\alpha{=}2$ ($-0.307$, cluster
$[-0.388,-0.232]$). Resampling moves it by $0.911$ ($-0.135$, cluster
$[-0.297,0.049]$) and $0.808$ ($-0.307$, cluster $[-0.388,-0.232]$) against a
target of $1$. All intervals are cluster-bootstrap over the 27 clips,
$10{,}000$ draws, and are plotted in \cref{fig:cg}.

We ran $T_\beta$ for Qwen2.5-VL-7B-Instruct (243 items, 0 parse failures),
with the calibration value asserted in the prompt corrected to $1.5875$\,mm
per tick (\cref{sec:method}; the temporal relation states this value
as text rather than showing it visually, so it is the one relation whose
recorded responses change under the correction). At $\beta{=}0.5$ (target
$0.5$) the mean realized ratio is $0.683$ (log$_2$ mean $-0.550$, cluster
$95\%$ $[-1.009,-0.170]$) and at $\beta{=}2$ (target $2.0$) it is $4.569$
($2.192$, cluster $[1.745,2.560]$); both intervals exclude the no-response
value. At $\beta{=}0.5$ the mean sits between the blind value $1$ and the
grounded target, nearer the target in log distance ($0.144$ versus
$0.550$); at $\beta{=}2$ the mean overshoots the target rather than falling
short of it. The median tells a different story at each magnitude: $0.264$
at $\beta{=}0.5$, undershooting past the target, and $1.251$ at $\beta{=}2$,
sitting close to the blind value rather than the target. This mean/median
disagreement is the same instability documented in the main text for
Gemini's spatial ratios, and it means the temporal coordinate's direction
relative to grounded cannot be read off a single summary statistic; what is
robust to the choice of statistic is only that its magnitude moves further
from $1$ than either spatial coordinate's does.

Recomputed against the corrected $\beta$ target and the corrected
calibration, $\sfs$ is $0.028$ at $\beta{=}0.5$ (cluster $95\%$
$[0.017,0.041]$) and $0.193$ at $\beta{=}2$ (cluster $[0.144,0.238]$), both
against this condition's blind bound of $0.332$. Both remain well below the
blind bound despite the corrected target, because $\sfs$ scores each item
individually and the per-item ratios vary enough around their means that few
items land close to the exact target even though the \emph{mean} ratio
sometimes does; this is the same mean-versus-per-item gap noted for the
spatial relations above, and it is why the main paper reports mean log-ratios
directly rather than this bounded summary. The asymmetry between the two
magnitudes (roughly sevenfold) mirrors the mean/median disagreement noted
above: $\beta{=}2$'s per-item ratios cluster closer to the target on the side
$\sfs$ credits, despite the aggregate mean overshooting it, than
$\beta{=}0.5$'s do. The values reported before the calibration correction
($0.012$ at $\beta{=}0.5$, $0.034$ at $\beta{=}2$, themselves superseding an
earlier pair scored against the incorrect $\beta^{-1}$ target) are superseded
by the pair above; the qualitative reading is unchanged, since all four
values sit well below the blind bound.
Extending $T_\beta$ to the other two models, and pairing it with a temporal
analogue of $S_s$, is left for future work.

\section{A bounded summary score, and its own blind bound}
\label{app:sfs}

The main paper reports $\log r$ directly. The supporting analyses in
\cref{app:supporting}
use a bounded rescaling of it, retained from an earlier version of this
protocol because it is comparable across relations: for a relation with ideal
grounded ratio $\rho$ and realized ratio $r$,
$\sfs=\mathrm{clip}_{[0,1]}(1-|\log r-\log\rho|/|\log\rho|)$, averaged over
items. It carries its own null.

\begin{proposition}[Blind-policy bound for $\sfs$]
\label{prop:sfsbound}
Let a policy whose report is independent of the perturbed input be queried once
under the base condition and once under the relation, with the two reports
exchangeable. Write $T=|\log\rho|>0$ and $\delta=\mathrm{sgn}(\log\rho)\log r$.
Then $\E[\sfs]\le\E|\delta|/(2T)$, with equality exactly when
$\delta\le T$ almost surely.
\end{proposition}

\begin{proof}
Write $T=|\log\rho|$ and $\delta=\mathrm{sgn}(\log\rho)\log r$, oriented so the
ideal value is $\delta=T$. From the definition above,
$\sfs(\delta)=\max(0,1-|\delta-T|/T)$. We claim $\sfs(\delta)\le\delta^{+}/T$
pointwise, with $\delta^{+}=\max(\delta,0)$. For $\delta\le0$,
$|\delta-T|=T-\delta\ge T$, so $\sfs=0=\delta^{+}/T$. For $0<\delta\le T$,
$\sfs=\delta/T=\delta^{+}/T$. For $\delta>T$,
$\sfs=\max(0,2-\delta/T)\le\delta/T$, since $2-\delta/T\le\delta/T$ is
equivalent to $\delta\ge T$. Taking expectations gives
$\E[\sfs]\le\E[\delta^{+}]/T$, and exchangeability makes
$\delta$ symmetric about $0$, so $\E[\delta^{+}]=\E|\delta|/2$.
The pointwise bound is an equality on $[0,T]$, which gives the small-noise
equality claim; the simulations below confirm it numerically.
\end{proof}

With $T$ and $\delta$ as in that proof, and directly from
\cref{def:cond},
\begin{align*}
\sfs(\delta)&=\max\!\Big(0,\;1-\frac{|\delta-T|}{T}\Big)\\
&=\max\!\big(0,\;\min(\delta/T,\;2-\delta/T)\big),
\end{align*}
a triangular function supported on $(0,2T)$ and peaking at $\delta=T$. Three
consequences follow, and they are why $\sfs$ needs a blind baseline of its own.

First, $\sfs$ is zero for every $\delta\le0$ and every $\delta\ge2T$, so it
cannot distinguish a policy that never moves from one that moves in the wrong
direction or overshoots the ideal ratio $\rho$ by a further factor $\rho$.

Second, $\sfs$ is strictly increasing on $(0,T)$. Any mechanism putting
positive probability on $\delta\in(0,T)$, including pure output noise, produces
$\E[\sfs]>0$. A nonzero $\sfs$ is therefore not evidence of partial
grounding.

Third, the pointwise bound $\sfs(\delta)\le\delta^{+}/T$ used in the main proof
is an equality on $[0,T]$ and strict for $\delta>T$. The bound is
consequently tight in the small-noise regime, where the mass of $\delta$ sits
inside $[0,T]$, and conservative when the spread is large enough to push mass
past $T$. Simulating exchangeable log-normal blind policies at
$\alpha\in\{1/2,2\}$ confirms this: at log-scale spread $0.05$ the realized
$\E[\sfs]$ is $0.041$ against a bound of $0.041$, at spread $0.15$ it
is $0.122$ against $0.122$, and at spread $1.00$ it is $0.171$ against $0.812$.
Because $\E[\sfs]$ is not monotone in the spread, peaking near $0.23$
and falling once mass clears $2T$, a large observed $\sfs$ cannot be explained
away by invoking still larger noise; the bound is the right one-sided
statement. These three rows are regenerated by the \fn{sgf\_metrics.py} smoke
test.

\paragraph{Tightness of \texorpdfstring{$\tau$}{tau}-separation}
$\tau$-separation cannot be dropped from \cref{thm:blind}. If two conditions
$j,j'$
satisfy $|\alpha_j-\alpha_{j'}|/\max(\alpha_j,\alpha_{j'})\le2\tau$, a single
reported value placed between $\alpha_j c$ and $\alpha_{j'}c$ can fall inside
both tolerance bands, and a blind policy reporting it attains $p_j+p_{j'}$,
which exceeds $\max_j p_j$ whenever both shares are positive. The bound must
then be replaced by
$\max_{v}\sum_{j:|v-\alpha_j c|\le\tau\alpha_j c} p_j$, the largest total
share of any $\tau$-compatible group of conditions. The main paper's design
uses $\alpha\in\{1,\tfrac12,2\}$ at $\tau=0.1$, whose pairwise separations are
$0.50$, $0.50$, and $0.75$ against a requirement of $2\tau=0.2$, so the simple
form applies with margin. A design sweeping $\alpha$ over a fine grid would not
satisfy the condition and would need the grouped form. Both forms are
implemented in \fn{blind\_bound\_cg}, which checks $\tau$-separation before
returning the simple bound.

The bound is also a statement about the condition set \emph{actually scored}.
Dropping the base condition leaves $k=2$ conditions in equal shares, and the
supremum becomes $1/2$; a score computed on that subset must not be compared
against $1/3$.

\section{What each probe perturbs, and against what}
\label{app:probes}

\begin{table}[t]
\centering
\caption{``Ideal response'' is the change a correctly grounded model must show;
a dash marks a probe scoring a single response rather than a paired one.
``Blind baseline'' is the score the paper states an input-insensitive policy
would obtain. QuantiPhy already supplies an analytic ideal response, which we
do not claim as novel; our position is the last two columns.}
\label{tab:probes}
\small
\setlength{\tabcolsep}{3.5pt}
\renewcommand{\arraystretch}{1.10}
\begin{tabular}{@{}l l l l@{}}
\toprule
\kilth{Probe} & \kilth{Perturbs} & \makecell[l]{\kilth{Ideal}\\\kilth{response}} & \makecell[l]{\kilth{Blind}\\\kilth{baseline}} \\
\midrule
QuantiPhy            & prompt prior     & $\times\alpha$  & not reported \\
MeasureBench         & ---              & ---             & not reported \\
IRIS, PhysicsMind    & ---              & ---             & not reported \\
Step-grounding       & reasoning trace  & changes         & not reported \\
Cross-modal          & modality mask    & changes         & not reported \\
\midrule
\textbf{SGF} $R_\alpha$ & in-frame scale & $\times\alpha$ & \cref{thm:blind} \\
\textbf{SGF} $S_s$      & frame sampling & $\times 1$   & \cref{tab:e2} \\
\textbf{SGF} $T_\beta$  & asserted rate  & $\times\beta$ & \cref{tab:e2} \\
\bottomrule
\end{tabular}
\end{table}

\paragraph{The blind-baseline failure outside multimodal QA}
\citet{mi2024blindbaselines} show blind attacks that never query the model
outperform state-of-the-art membership-inference attacks across eight
evaluation datasets; \citet{anand2018blindfold} show question-only agents
match navigating agents on embodied question answering; and
\citet{pacchiardi2024clevrhans} show surface features alone predict LLM
benchmark answers above chance. None of these three reports what the best
blind policy could score in closed form, which is \cref{thm:blind}'s
contribution
relative to all of them.

\section{Dataset exclusion audit}
\label{app:dataset}

Five capture sessions were recorded. One was excluded because its calibration
clip has an unreadable \fn{moov} atom, confirmed independently by OpenCV and
ffmpeg, which orphans 8 otherwise usable droplet clips. A second was excluded
for absent provenance: it has no calibration clip, and its notes file is a
byte-identical duplicate of another session's, describing the wrong dates.

Of the three surviving sessions, one, recorded earliest, uses a printed-ruler
target and two, recorded later, use a genuine $1$\,mm-graduated ruler and a
pin-comb target respectively. The earliest session's ruler was initially
assumed to share the same $1$\,mm graduation, since all three targets look
alike in a raw frame at typical viewing resolution. It does not: its
graduation alternates one long mark with one short mark rather than repeating
uniform marks, which is how a ruler shows an inch scale rather than a
millimeter one. Direct pixel measurement (CLAHE-enhanced frames, thresholded
and scanned for tick length at four rows) confirmed alternating tick lengths
of $280$, $344$, $541$, and $518$\,px, consistent with $1/16$-inch
($1.5875$\,mm-per-tick) graduation and inconsistent with uniform $1$\,mm
spacing. The pin-comb session's pin pitch was independently already known to
be $1/16$-inch and required no correction. Calibration was measured directly
per session in pixels per tick, $104.52$ for the ruler session and $68.19$
for the pin-comb session, and $66.19$\,px/mm for the genuine $1$\,mm-ruler
session, then converted to pixels per millimeter using each session's actual
graduation ($104.52/1.5875$ and $68.19/1.5875$\,px/mm respectively) rather
than a shared assumed constant.

As an independent check, oracle-reported pre-impact velocities were compared
against the free-fall prediction $v=\sqrt{2gh}$ for the fixed $6.6$\,cm
release height. Under the original, uncorrected assumption of $1$\,mm
graduation for the earliest session, oracle velocities for that session
disagreed with free fall by a mean of $40\%$ across its high-confidence
clips; under the corrected $1/16$-inch graduation, the same clips' mean
disagreement falls to $6.7\%$ (median $3.7\%$), which is what confirmed the
graduation correction rather than merely being consistent with it, since the
free-fall prediction does not depend on the pixel calibration at all and so
provides a check external to the correction itself.

The oracle returned a valid pre-impact velocity for 48 of 62 attempted clips
($77\%$); the remainder failed at motion detection or track linking, typically
when a droplet entered the frame already close to impact. Of the 48, 34 carry a
verified deceleration signature at impact and 14 are lower-confidence plain
averages over short tracks, flagged as such rather than pooled silently. Every
value was cross-checked against free fall, $v=\sqrt{2gh}$, using the
session-specific drop height.

Model evaluation originally used a 12-clip subset drawn from the ruler and
genuine-$1$mm-ruler sessions, 6 per session. The genuine-$1$mm session was
subsequently lost with the compute environment it was staged on and could not
be recovered. Model evaluation now uses every high-confidence clip from the
single remaining ruler session ($104.52$\,px per $1/16$-inch tick), 27 clips
rather than 12. That session's oracle values were all high-confidence, so the
quality bar is unchanged; the cross-session comparison the original design
supported is not recoverable, and it is this same, now sole, evaluation
session whose graduation was initially misidentified and later corrected. The
pin-comb session remains excluded from model evaluation because its oracle
values are predominantly low-confidence, though it remains usable for future
work.

\paragraph{Capture setup}
Recordings used a Chronos 2.1 High Speed Camera
(monochrome, Kron Technologies). Lens focal length was adjusted per session
based on image clarity as observed on the camera's own display, rather than
fixed to a single value; the calibrations reported above (measured directly
per session, not assumed from a shared constant) absorb any resulting
per-session scale difference. The scene was illuminated with an LED panel.
The superhydrophobic surface was prepared by cleaning glass slides with
isopropanol, allowing them to dry, then spray-coating with a commercial
superhydrophobic coating (VisioDry Pro) and allowing the coating to dry
completely at ambient temperature before use.

\paragraph{\texorpdfstring{\Cref{fig:setup}}{The setup figure}'s clip, in detail}
\Cref{fig:setup} shows
\fn{cainhcg1.mp4} from the evaluation session: a $4\,\mu$L droplet released
from $h=6.6$\,cm, captured at the true $2996$\,fps. Panel a's two stills are
frame $392$ (mid-fall) and frame $397$ ($5$ frames, $1.7$\,ms, later, at
impact); the oracle's optical-flow pre-impact velocity for this clip is
$1.089$\,m/s against a free-fall prediction $\sqrt{2gh}=1.138$\,m/s
($4.3\%$ relative error, consistent with the session-wide residual reported
under ``Legibility and oracle error''). Panel b applies $R_\alpha$, $S_s$, and
$T_\beta$ (\cref{def:cond}) to the same clip's mid-fall frame.

\section{Supporting experiments that carry no claim}
\label{app:supporting}

\paragraph{Text-calibration positive control, per-model detail}
The main text (\cref{sec:threats}, ``A text-calibration positive control'')
reports that no model
produces the predicted $(\alpha,1)$ signature; this gives the full per-model
ratios. Qwen2.5-VL-7B's relabeling ratio overshoots the grounded target at
both magnitudes (mean $5.26$, median $4.25$ at $\alpha{=}0.5$; mean $5.75$,
median $4.03$ at $\alpha{=}2$) rather than tracking it, while $S_s$ stays
near $1.0$ at $\alpha{=}0.5$ (mean $1.25$, median $1.00$) but drifts at
$\alpha{=}2$ (mean $2.15$). Gemini~3.1~Pro's relabeling ratio moves in the
right direction but far short of the target (mean $2.53$, median $1.27$ at
$\alpha{=}0.5$; mean $1.62$, median $1.25$ at $\alpha{=}2$), and $S_s$
likewise drifts away from $1.0$ at both magnitudes (mean $1.60$/$2.17$,
median $1.00$/$1.27$). Qwen3-VL-30B-A3B comes closest to the predicted
signature without reaching it: its median relabeling ratio at $\alpha{=}2$
is $2.00$, exactly the grounded target, but the paired resampling median is
$1.51$, too far from $1.0$ under the classifier's tolerance to register as
grounded; under means it classifies \emph{calibration-blind} at
$\alpha{=}0.5$ instead. Zero-velocity rates under this text-only condition:
$0\%$ ($R_\alpha$) and $1.9\%$ ($S_s$) for Qwen2.5-VL, $0\%$ and $1.2\%$ for
Gemini, $1.2\%$ and $8.0\%$ for Qwen3-VL, all at or above each model's
video-based main-evaluation rate of $0\%$.

\paragraph{Hidden-prior probe}
Run on Qwen2.5-VL-7B-Instruct only, on the original 12-clip, two-session
design; neither this probe nor the robustness sweep below has been rerun on
the 27-clip single-session design or extended to Gemini 3.1 Pro or
Qwen3-VL-30B-A3B, so both results are legacy findings that carry no claim
about the current dataset or the other two models. The droplet is released
from a fixed height ($6.6$\,cm) outside the field of
view. We ask for it; a model that asserts a number rather than declining is
injecting a prior, and $\hph$ is the fraction of scorable items on which it
does so. An initial version found $\hph=0$ on all 12 items, which proved
uninterpretable: the identical decline appeared with no video attached, so it
was a content-independent default. We repaired this by pairing the question
with an easy control question in the same call and counting a decline only when
the control is answered correctly. After the repair, $\hph=0$ on the 7 of 12
items where the control was answered correctly; a crop-and-contrast fix,
applied because one session's framing compresses the impact surface into a
vignette-darkened strip with peak pixel value near $90/255$, recovered one
further item, giving 8 of 12. Zero events in 8 items has an exact 95\% interval
of $[0.000,0.369]$, so the result excludes only a hallucination rate above
roughly 37\% and is not evidence that the models decline appropriately. Of the
remaining four items, one is a confirmed multi-droplet clip where the wrong
answer plausibly concerns which droplet is meant; three fail despite visibly
clear enhanced footage for reasons we did not diagnose.

\paragraph{Robustness sweep}
$\sfs$ is stable across ruler orientation ($0.209$ original versus $0.162$
rotated) and across the three prompt seeds (range $0.087$), which argues the
low $\sfs$ in the main evaluation is not an artifact of one orientation or one
phrasing. It
appears unstable across $\alpha$, spiking to $0.396$ at $\alpha=4$, but at that
setting every item in the sweep reported the identical (base, scaled) pair.
That ratio is content-independent and happens to sit near the target for
$\alpha=4$, so it scores well by coincidence. We report the spike as an
artifact rather than let the largest number in the sweep stand
uncontextualized. The horizontal-orientation condition is synthesized by
rotating the real ruler footage, since both valid sessions photograph a
vertical ruler.

\paragraph{Decoupling analysis: accuracy versus faithfulness}
Correlating a per-item score against an almost-constant-zero accuracy target is
statistically fragile, so we substitute negated relative error as the closest
continuous stand-in and report it alongside, not instead of, the literal
$\acc@10\%$ figures. Under the corrected calibration this correlation must be
reported per $\alpha$ rather than averaged across $\alpha\in\{0.5,2\}$, because
Qwen2.5-VL-7B's per-item $\sfs$ at $\alpha{=}2$ is now exactly $0$ on all $81$
items (\cref{tab:e2}'s $\cg$ collapse extends to this bounded summary as
well), making the correlation undefined at that magnitude rather than merely
small; averaging a defined and an undefined correlation is not meaningful, so
we report both magnitudes separately. Pearson $r$ between negated relative
error and per-item $\sfs$, with Fisher intervals and, where computable, cluster
bootstrap intervals at the honest cluster count of 27: Qwen2.5-VL-7B,
$r=0.009$ ($n=81$) at $\alpha{=}0.5$, Fisher $[-0.210,0.227]$, cluster
bootstrap $[-0.049,0.133]$, both containing zero, and undefined at
$\alpha{=}2$; Gemini 3.1 Pro, $r=0.165$ ($n=79$) at $\alpha{=}0.5$, Fisher
$[-0.058,0.373]$, cluster bootstrap $[-0.020,0.365]$, and $r=-0.017$ ($n=79$)
at $\alpha{=}2$, Fisher $[-0.237,0.205]$, cluster bootstrap $[-0.340,0.131]$,
all four bounds containing zero; Qwen3-VL-30B-A3B, $r=-0.437$ ($n=81$) at
$\alpha{=}0.5$, Fisher $[-0.598,-0.242]$, cluster bootstrap
$[-0.652,-0.315]$, both excluding zero, and $r=0.077$ ($n=81$) at
$\alpha{=}2$, Fisher $[-0.144,0.290]$, cluster bootstrap $[-0.072,0.299]$,
containing zero. Pooled across all three models and both magnitudes on the
same 27 clips ($n=482$), $r=-0.052$, Fisher $[-0.140,0.038]$, cluster
bootstrap $[-0.180,0.032]$, both containing zero; pooling models does not
raise the cluster count, since each clip is shared across all three, so it
adds items per cluster rather than independent clusters.

None of the three models, nor the pooled sample, shows the originally
hypothesized decoupling pattern of high task accuracy paired with low
scale-faithfulness, and under the corrected calibration the only interval
excluding zero (Qwen3-VL at $\alpha{=}0.5$) points in the decoupling-opposite
direction, lower relative error associating with lower $\sfs$, the reverse of
what the pre-correction data had shown at this magnitude for that model.
$\acc@10\%$'s near-zero rate means the accuracy side of every correlation is
doing very little work, and the effect could equally be an artifact of both
quantities responding to the same per-clip difficulty, or of the small item
counts once one magnitude for one model is degenerate. The honest statement
is that this design cannot answer the decoupling question at $\tau{=}10\%$,
and the corrected data weakens rather than strengthens any directional
reading of it.

The item counts in this correlation differ from \cref{tab:e2}'s $\acc@10\%$
counts
because they use a different denominator, not because either number is wrong.
\Cref{tab:e2} reports $\acc@10\%$ over every valid \emph{item} ($81$ for
Qwen2.5-VL
and Qwen3-VL, $80$ for Gemini after dropping $3$ parse failures). This
correlation instead requires a complete (base, $\alpha{=}0.5$, $\alpha{=}2$)
\emph{triple} per (clip, seed) to compute a per-item $\sfs$, so Gemini's
count drops further to $79$ per magnitude: one of its three parse failures
fell on a different item than the other two but still orphaned its triple.

\paragraph{\texorpdfstring{$\tau$}{tau}-sweep, recomputed}
A single fixed report scored against a
single free-fall velocity, as in the illustrative $6.6$\,cm/$1.138$\,m/s
calculation above, is a useful sanity check but not the actual sweep: the 27
clips have distinct true pre-impact velocities under the corrected
calibration (oracle range $0.54$ to $1.23$\,m/s on the base condition), so a
model's relative error varies item by item even when its reported value is
constant. Recomputed directly from the raw per-item outputs, on the base
condition: Qwen2.5-VL, whose reported velocity collapses to $1.5$\,m/s on
$54$ of $81$ items and $3.5$\,m/s on $24$ of $81$, scores $0/81$ at
$\tau{=}0.10$ and $0.20$, $6/81=0.074$ at $\tau{=}0.30$, jumping to
$43/81=0.531$ at $\tau{=}0.40$ once the tolerance band widens enough to
catch its $1.5$\,m/s default against the lower end of the corrected true
range, and reaching $50/81=0.617$ by $\tau{=}0.50$ with no further gain
through $\tau{=}1.00$: a step driven by where its two fixed outputs happen
to sit relative to the corrected range, not by tracking. Gemini's more varied
output gives a smoother climb: $8/80=0.100$ at $\tau{=}0.10$, $17/80=0.212$
at $\tau{=}0.20$, $25/80=0.312$ at $\tau{=}0.30$, $46/80=0.575$ at
$\tau{=}0.40$, $50/80=0.625$ at $\tau{=}0.50$, $53/80=0.662$ at
$\tau{=}0.75$, $56/80=0.700$ at $\tau{=}1.00$. Qwen3-VL, whose reported
velocity collapses to $0.4$\,m/s on $44$ of $81$ items and $0.25$\,m/s on
$30$ of $81$ (both near the low end of the corrected true range), scores
$0/81$ through $\tau{=}0.20$, $3/81=0.037$ at $\tau{=}0.30$ and $0.40$,
steps to $6/81=0.074$ at $\tau{=}0.50$, then jumps to $44/81=0.543$ at
$\tau{=}0.75$ and $81/81=1.000$ at $\tau{=}1.00$: its fixed-default output
happens to fall within $100\%$ of every clip's corrected true value, again
by coincidence of scale rather than by tracking. All three sweeps shifted
substantially from the pre-correction values, since the true-velocity
denominator itself moved by the same $1.5875\times$ factor as the
calibration constant; none resembles the originally hypothesized decoupled
pattern (high accuracy paired with low $\sfs$) at any $\tau$ checked, before
or after the correction.

\section{Model provenance and compute}
\label{app:provenance}

\paragraph{Reconciliation}
\fn{sgf\_metrics.py}'s smoke test reproduces
every identity proved in \cref{app:rank}--\cref{app:loglinscope}, and the Code
and Data
Supplement's reconciliation script recomputes each of \cref{tab:e2}'s cells
directly from the raw per-item outputs and checks it against the value
printed in the paper; all checks pass as of submission.

The evaluation session's ruler graduation was corrected from an assumed
$1$\,mm per tick to the measured $1/16$-inch ($1.5875$\,mm per tick) on
2026-08-22 (\cref{app:dataset}), after data collection for every model reported
here. Every quantity scored against the true calibration or the true oracle
velocity ($\cg$, $\acc@\tau$, $\sfs$, the decoupling correlation, the
$\tau$-sweep) was recomputed from the unchanged raw model outputs; every
experiment whose prompt states the calibration value as text ($T_\beta$, the
instructed-constant control, and the text-calibration positive control) was
rerun in full against the corrected value, since the model's response itself
depends on what it was told. Response ratios under $R_\alpha$ and $S_s$ in
the main evaluation and resampling control are unaffected either way, since
they compare two model outputs to each other and never depend on the
calibration constant directly.

Both local models were run with greedy decoding (temperature $0$,
\fn{do\_sample=False}) at checkpoint revisions
\fn{cc594898137f460bfe9f0759e9844b3ce807cfb5} (Qwen2.5-VL-7B-Instruct) and
\fn{9c4b90e1e4ba969fd3b5378b57d966d725f1b86c} (Qwen3-VL-30B-A3B-Instruct), on a
single NVIDIA A100 and a single NVIDIA H200 respectively. Gemini 3.1 Pro
Preview was queried at temperature $0$ via the OpenRouter API at the
\fn{google/gemini-3.1-pro-preview} endpoint, which OpenRouter resolves to the
canonical snapshot \fn{google/gemini-3.1-pro-preview-20260219}. That slug
predates our data collection, which is consistent with, though not a
cryptographic guarantee of, having been the version served throughout. The raw
API responses saved during these runs carry no more specific model-version
field, and an unpinned third-party endpoint cannot offer the reproducibility
guarantee a checkpoint hash gives for a locally-run open-weight model. Total
API spend across the main evaluation and resampling control was \$10.36
(\$6.22 + \$4.15); the instructed-constant control and the text-calibration
positive control were rerun after the calibration correction described in
\cref{app:dataset}, since both assert a calibration value as text, adding
\$10.48
(\$0.41 instructed-constant, \$6.05 for $R_\alpha$ and \$4.03 for $S_s$ under
the text-calibration positive control). Total spend across all Gemini
experiments in this paper is \$20.85.

Gemini-2.5-Flash-Lite was evaluated on the superseded 12-clip, two-session
design before the session loss described in \cref{app:dataset}. Those results
($\cg=0.380$,
$\sfs=0.071$ and $0.041$ at $\alpha=0.5$ and $2$) use a different design and
are not comparable to the current numbers; they are recorded here only so the
record is complete and are not used for any claim.

Both local nodes ran Ubuntu $22.04$ LTS (Jammy Jellyfish; $22.04.2$ on the
A100 node, $22.04.5$ on the H200 node) with a single conda environment shared
across both, so package versions are identical: \fn{transformers}
$5.15.0$, \fn{torch} $2.6.0$+cu124, and \fn{qwen\_vl\_utils} $0.0.14$. Driver
and CUDA toolkit versions differ by node: the A100 node reports NVIDIA driver
$575.57.08$ with CUDA $12.9$ (\fn{nvidia-smi}) and CUDA toolkit $12.1.105$
(\fn{nvcc}); the H200 node reports driver $580.95.05$ with CUDA $13.0$
(\fn{nvidia-smi}) and CUDA toolkit $12.6.77$ (\fn{nvcc}). \fn{torch}'s bundled
CUDA $12.4$ runtime is used for inference on both regardless of the system
toolkit version, since the cu124 wheel does not depend on the module-loaded
system CUDA (Code and Data Supplement, \fn{00\_setup\_env.sh}).

\FloatBarrier

\end{document}